%% file: main_aistats26.tex
\documentclass[twoside]{article}

\usepackage[accepted]{aistats2026}

\usepackage[round]{natbib}
\renewcommand{\bibname}{References}
\renewcommand{\bibsection}{\subsubsection*{\bibname}}

\usepackage{gav_style}

\usepackage{hyperref}
\usepackage{url}
\usepackage{wrapfig}
\usepackage{multirow}   %
\usepackage{graphicx}   %
\usepackage[capitalise,noabbrev]{cleveref}
\usepackage{mathtools}

\usepackage{xr}

\usepackage[acronym]{glossaries}
\glsdisablehyper

\newacronym{mti}{MTD}{mutual transport dependence}
\newacronym{bed}{BED}{Bayesian experimental design}
\newacronym{oed}{OED}{optimal experimental design}
\newacronym{eig}{EIG}{expected information gain}
\newacronym{mi}{MI}{mutual information}
\newacronym{etti}{TTD}{target transport dependence}
\newacronym{edti}{DTD}{expected data transport dependence}

\newcommand{\EIG}{\mathcal{I}} %
\newcommand{\sm}{\mathrm{sigmoid}} %

\newcommand{\jointgain}{\TT_c}
\newcommand{\thetagain}{\TT_c^{(\theta)}}
\newcommand{\ygain}{\TT_c^{(y)}}

\begin{document}

\runningtitle{A Geometric Approach to Optimal Experimental Design}

\runningauthor{Gavin Kerrigan, Christian A. Naesseth, Tom Rainforth}

\twocolumn[

\aistatstitle{A Geometric Approach to Optimal Experimental Design}

\aistatsauthor{Gavin Kerrigan$^\star$ \And Christian A. Naesseth \And Tom Rainforth}
\aistatsaddress{University of Oxford \And University of Amsterdam \And University of Oxford}
]

\begin{abstract}
\looseness=-1
    We introduce a novel geometric framework for optimal experimental design (OED). Traditional OED approaches, such as those based on mutual information, rely explicitly on probability densities, leading to restrictive invariance properties.
    To address these limitations, we propose the mutual transport dependence (MTD), a measure of statistical dependence grounded in optimal transport theory which provides a geometric objective for optimizing designs. 
    Unlike conventional approaches, the MTD can be tailored to specific downstream estimation problems by choosing appropriate geometries on the underlying spaces. We demonstrate that our framework produces high-quality designs while offering a flexible alternative to standard information-theoretic techniques.
\end{abstract}

\section{Introduction}
\label{sect:intro}
\glsresetall

\looseness=-1
Effective experimental design is central to a wide range of scientific and industrial applications \citep{kuhfeld1994efficient, park2013bayesian, melendez2021designing}. 
Many such problems require a principled, model-based, approach, wherein we utilize a model over possible experimental outcomes to directly optimize our design decisions.
This can be particularly effective in \emph{adaptive} design settings, where frameworks like sequential \gls{bed} allow us to iterate between using our model to make design decisions and updating our model with the collected data~\citep{degroot1962uncertainty,mackay1992information,sebastiani2000maximum,rainforth2024modern}.
Many of these approaches are grounded in \emph{information theory}, where the value of an experiment is quantified using the values of an associated probability density.

For instance, a popular and principled approach in the \gls{bed} literature is optimizing the \gls{mi} \citep{lindley1956measure,lindley1972bayes,bernardo1979expected}
\begin{align}
    \EIG(d) &= {\sf{KL}}\left[ p(\theta, y \mid d) \mid\mid p(\theta) p(y \mid d) \right] \label{eqn:eig_1} \\
            &= \E_{p(\theta, y \mid d)}\left[ \log \left( \frac{p(\theta, y \mid d)}{p(\theta)p(y \mid d)} \right) \right]
\end{align}
\looseness=-1
where $p(\theta)$ is the prior over the quantity of interest $\theta$, and $p(y \mid \theta, d)$ models the experiment outcome $y$ under design $d$. Notably, the \gls{mi} is an expectation of log-density ratios, and is thus a \textit{unitless} quantity. 
Consequently, the \gls{mi} has strong invariance properties: any injective transformation of $\theta$ or $y$ leaves $\EIG(d)$ unchanged. 

\looseness=-1
We highlight that this invariance, while often attractive, can also be detrimental for \gls{oed}. Experimental goals should be defined in terms of downstream errors~\citep{lindley1972bayes} and many common error metrics, such as mean squared error between true and estimated parameters, are inherently geometric and dependent on the space in which they are measured. Thus, the \gls{mi}, being purely informational, cannot be naturally aligned with task-specific error metrics and has no mechanism by which it may be targeted to a particular geometric distance on predictions: it implicitly assumes errors are measured by log loss. 
In turn, this inflexibility can be problematic for various applications.
For example, in financial settings, we often inherently care about the variance in future returns and not the entropy, noting that the two can take arbitrary values with respect to one another.

\begin{figure*}[!t]
    \centering
    \includegraphics{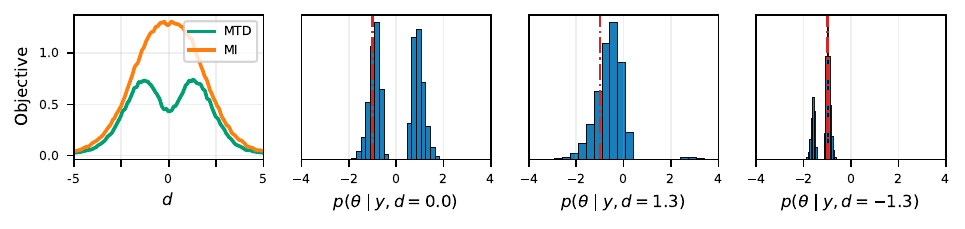}
    \caption{Comparison of \gls{mi} and \gls{mti} on the 1D source location-finding problem, with the true source at $\theta=-1$ (dashed red line). MI is maximized at the origin, reflecting the prior mode, whereas \gls{mti} is maximized near $d \approx \pm 1.3$. The posterior for $d=0$ is bimodal, while for $d=\pm 1.3$ it becomes unimodal or sharply concentrated. In practice, $d$ is optimized directly, and \gls{mti} breaks the posterior symmetry even when initialized unfavorably.}
    \label{fig:leading}
\end{figure*}

The MI also poses several practical challenges~\citep{rainforth2024modern}. Foremost, it is a doubly intractable quantity, in general requiring a nested estimation~\citep{rainforth2018nesting} of either the posterior $p(\theta \mid y, d)$ or marginal ${p(y \mid d)}$. Second, in implicit settings, where the likelihood can only be sampled but not evaluated, the MI faces additional challenges due to its explicit reliance on densities. While there exist approaches for estimating the MI in implicit settings \citep{kleinegesse2019efficient, ivanova2021implicit}, these amount to learning the unknown density ratio, a task that becomes increasingly difficult in high dimensions. 
This problem is also not specific to MI, with the core BED formalism inherently having nested dependence on the posterior~\citep{lindley1972bayes,chaloner1995bayesian}.

We address these shortcomings by proposing the \emph{\acrfull{mti}}, a novel class of geometric criteria for experimental design. The \gls{mti} measures the dependency between $\theta$ and $y$ in terms of an optimal transport discrepancy \citep{villani2008optimal, feydy2019interpolating} between the joint distribution $p(\theta, y\mid d)$ and its product of marginals $p(\theta) p(y \mid d)$. %
The \gls{mti} depends explicitly on a choice of \emph{sample-level} cost function, enabling practitioners to encode domain knowledge or downstream objectives directly into the design criterion. This cost function can be defined either directly or through a transformation of the underlying space, creating a family of flexible design criteria.

Moreover, the \gls{mti} offers practical benefits. It does not involve any nested expectations, can be estimated without likelihood evaluations, and can be directly optimized using gradient-based methods whenever differentiable sampling of $p(y \mid \theta, d)$ is possible. This makes it particularly well-suited to simulation-based or implicit scenarios where the likelihood $p(y \mid \theta, d)$ is unknown or intractable, an increasingly common setting in OED \citep{kleinegesse2019efficient, kleinegesse2021gradient, ivanova2021implicit, encinar2025deep},

Optimizing \gls{mti} results in qualitatively different design behaviour than \gls{mi} (see \cref{fig:leading}). We illustrate the effectiveness of \gls{mti}-optimal designs on standard experimental design benchmarks, comparing directly with \gls{mi}-based designs. Our results show that the \gls{mti} can outperform traditional information-theoretic approaches, while also allowing experimenters to tailor the geometry of the underlying spaces to the problem at hand. In sum, our framework introduces a principled and practical new class of criteria for optimal experimental design, overcoming the rigidity of information-theoretic methods and enabling experiments that better reflect real-world objectives.

\section{Background and Notation}
\label{sect:background}

\looseness=-1
We use $\theta \in \Theta$ to represent the unknown target quantity of interest that we wish to learn about through our experiments. This could correspond to a real world quantity, model parameters, or something abstract like downstream predictions.
We further use $d \in \DD$ to represent an experimental design and $y \in \YY$ to represent an observed outcome of an experiment. 
Akin to standard BED approaches, we specify a prior $p(\theta)$ representing our beliefs about $\theta$ before performing any experiments and a likelihood $p(y \mid \theta, d)$ capturing the data generating process. %
Our goal is now to select $d$ in a way that will allow us to best estimate $\theta$ once the experiment's outcome is observed.

\subsection{OED with Mutual Information}
\label{subsect:bed_mi}

From an information-theoretic point of view, it is natural to seek designs which result in data $y$ that reduces our uncertainty about the unknown quantity $\theta$. That is, we consider the reduction in entropy \citep{lindley1956measure}
\begin{align} 
    \EIG(d) &= \E_{p(y \mid d)}\left[ {\sf H}[\theta] - {\sf H}[\theta \mid y, d] \right] \label{eqn:eig_entropy}
\end{align}
which can straightforwardly be shown to yield the mutual information \eqref{eqn:eig_1}. The design choice is $d^* = \argmax_{d \in \DD} \EIG(d)$, maximizing the mutual information. %
In the context of OED, $\EIG(d)$ is often called the \textit{expected information gain} (EIG).

In all but the simplest cases, computing $\EIG(d)$ is non-trivial, as it requires estimating both the outer expectation and the integrand (i.e., either the posterior $p(\theta \mid y, d)$ or marginal likelihood $p(y \mid d)$). Often, one resorts to nested estimators like nested Monte Carlo (NMC) \citep{rainforth2018nesting}, or variational approaches \citep{foster2019variational,foster2020unified}. Importantly, many techniques assume that the likelihood $p(y \mid \theta, d)$ is known explicitly, where the corresponding distribution can be evaluated pointwise.

\looseness=-1
OED is often particularly useful in adaptive scenarios, where experiments are designed sequentially based on the data $h_{t} = \{ (d_k, y_k) \}_{k=1}^{t}$ gathered in previous trials. 
In this setting, we replace the prior $p(\theta)$ with our updated beliefs using the posterior $p(\theta\mid h_{t})$, and consider the \emph{incremental} MI
\begin{equation} \label{eqn:seq_eig}
    \EIG^{(t+1)}(d) = {\sf KL}\left[ p(\theta, y | d, h_{t}) || p(\theta | h_{t}) p(y | d, h_{t}) \right].
\end{equation}

We refer to \citet{chaloner1995bayesian, rainforth2024modern} for more comprehensive surveys of BED.

\subsection{Optimal Transport}
\label{subsect:background_ot}

Optimal transport (OT) is a mathematical toolkit which allows us to compare two arbitrary probability distributions %
$p(x)$ and $q(x')$ in terms of the amount of work required to transform one distribution into the other \citep{villani2008optimal, peyre2019computational}. %

In the Kantorovich formulation of OT \citep{kantorovich1942translocation}, we specify a non-negative cost function %
$c(x, x') \geq 0$ encoding the cost of transporting a unit of mass from location $x$ to $x'$. A \textit{coupling} is a joint distribution $\gamma(x, x')$ whose marginals are $p, q$ respectively. %
We write $\Pi(p, q)$ for the set of all valid couplings. Given a coupling $\gamma \in \Pi(p, q)$, its associated cost is
\begin{equation} \label{eqn:kantorovich_objective}
    K_c(\gamma) = \int c(x, x') \d \gamma(x, x') = \E_{\gamma(x, x')}\left[ c(x, x') \right]
\end{equation}
which can be interpreted as the average sample-level transport cost under this coupling. An optimal coupling minimizes this cost, %
and we write 
\begin{equation} \label{eqn:ot_cost}
    {\sf OT}_c\left[p, q \right] = \min_{\gamma \in \Pi(\mu, \nu)}K_c(\gamma)
\end{equation}
for the minimum value attained. In short, $c(x, x')$ defines a sample-level cost function and ${\sf OT}_c[p, q]$ is the associated optimal transport discrepancy.%

\section{Experimental Design through Mutual Transport Dependence}
\label{sect:mtd}

The \gls{mi} is an inherently density-based objective, where the value of an experiment is considered a purely informational quantity that is based directly on the \emph{density} of the posterior, rather than the actual \emph{values} that $\theta$ and $y$ can take. As such, the \gls{mi} is unable to incorporate properties of the underlying sample spaces $\Theta, \YY$ into its design objective, such as an error metric on $\theta$. 

This induces strong invariance properties in the MI \citep[Theorem~3.7]{polyanskiy2025information}, such that it remains fixed under injective transformations of $\theta$ and/or $y$.
Thus, if we are interested in measuring errors in terms of some transformation $\phi = f(\theta)$, the optimal design remains fixed with changes in $f$, even though our intuitive notion of error itself will change. For example, while $\theta$ may be the natural variables for parametrizing a model, we may be interested in a one-to-one transformation of $\theta$ which is more interpretable. Under \gls{mi}, these scenarios are indistinguishable: any unit of entropy reduction is equally valuable, regardless of its practical implications.

\looseness=-1
These shortcomings motivate the need for a practical objective which is fully defined in terms of \textit{geometric} notions on the underlying sample spaces. In other words, we seek a criterion which is determined by the values taken on by random variables themselves.

\subsection{Mutual Transport Dependence}
\label{subsect:discrepancies}

One interpretation of \gls{mi} is that it measures the KL divergence between the joint, $p(\theta, y | d)$, under which $\theta$ and $y$ are dependent, and the product of marginals, $p(\theta)p(y | d)$, under which they are independent. 
The KL, though, depends only on density ratios and is thus unsuitable for a geometric measure of information.

To provide a geometric, sample-based criterion, we propose to instead measure the dependency between $\theta$ and $y$ via an optimal transport dependency between the same joint and product of marginals. Optimal transport, relying on an expected sample-level cost function, is a natural choice for such a geometric measure as it allows customisation of our notion of distance through the cost function. 
This yields our proposed \acrfull{mti}, defined as follows.\footnote{We work under standard measure-theoretic assumptions regarding the spaces $\Theta, \YY$ as well as the cost $c$, which we discuss in Appendix \ref{sect:appendix_mtd_details}.}

\begin{definition}%
    \label{def:joint_gain}
    Given a cost function $c: \Theta^2 \times \YY^2 \to \R_{\geq 0}$, the {\upshape{\acrlong{mti}}} (\gls{mti}) is
    \begin{equation}\label{eqn:transport_info}
    \begin{aligned}
    \jointgain(d) :&= {\sf OT}_c\left[p(\theta,y \mid d), p(\theta)p(y \mid d)\right] \\
    &= \min_{\gamma \in \Pi\left(p(\theta, y \mid d), p(\theta)p(y \mid d)\right)}  
       \E_{\gamma}\left[ c(\theta, y, \theta', y') \right].
    \end{aligned}
    \end{equation}
\end{definition}

Note that the \gls{mti} depends only on expectations with respect to our model, with no direct dependency on the density functions themselves.
This is critically important from a computational and scaling perspective, as there is no need to perform (nested) estimation of the posterior or marginal likelihood densities, which is typically challenging, especially in high dimensions.

High values of $\jointgain(d)$ indicate that the parameter and data are strongly coupled under this design, and conversely, $\jointgain(d)=0$ if and only if $\theta$ and $y$ are independent under $d$. As $\jointgain(d)$ can be also seen as the average $(\theta, y)$ displacement under the optimal transport plan, measured by the cost $c(\theta, y, \theta', y')$, it is a \textit{geometric} notion, and by choosing the cost function $c$ in an appropriate way the experimenter can incorporate downstream preferences directly into the experimental design process.

\subsection{Alternative Transport-Based Measures}
\label{subsect:alternatives}
The \gls{mi}, being the KL between the joint and product of marginals, can be equivalently understood as an expected KL discrepancy on either $\Theta$ or $\YY$ alone:
\begin{align} 
    \EIG(d) 
    &= \E_{p(\theta)} {\sf KL}\left[ p(y \mid \theta, d) \mid\mid p(y \mid d) \right] \label{eqn:eig_kl_y} \\
    &= \E_{p(y \mid d)} {\sf KL}\left[ p(\theta \mid y, d) \mid \mid p(\theta) \right] \label{eqn:eig_kl_theta}.
\end{align}
This equivalence does not translate to the \gls{mti}, as the transport distances in $\Theta$ and $\YY$ are inherently different.
However, we can derive transport-based analogues of these interpretations of the MI as well by considering costs defined on $\Theta$ or $\YY$ alone as follows.

\begin{definition}%
    \label{def:expected_theta_gain}
    Given a cost function $c: \Theta^2 \to \R_{\geq 0}$, the {{\upshape{\gls{etti}}}} is
    \begin{equation}
        \thetagain(d) := \E_{p(y \mid d)} \left[{\sf OT}_c\left[ p(\theta \mid y, d), p(\theta) \right]\right].
    \end{equation}
\end{definition}

\begin{definition}
    \label{def:expected_y_gain}
    Given a cost function $c: \YY^2 \to \R_{\geq 0}$, the {\upshape{{\gls{edti}}}} is
    \begin{equation} \label{eqn:wbed_y}
    \ygain(d) := \E_{p(\theta)} \left[ {\sf OT}_c\left[ p(y \mid \theta, d), p(y \mid d) \right]\right].
    \end{equation}
\end{definition}

\looseness=-1
In some ways, these transport dependencies are perhaps more intuitive than the \gls{mti}, $\jointgain(d)$, which relies on a cost defined on the joint space $\Theta \times \YY$. 
In particular, given that our ultimate goal is the estimation of $\theta$, the \gls{etti} arguably provides the more natural of the three objectives as it is directly measuring changes in beliefs in the space of $\Theta$.

\looseness=-1
However, unlike the \gls{mti}, these objectives have reintroduced a nesting into the problem as the maximization over couplings is now inside an expectation. This makes them, at least in principle, computationally more problematic as 
they appear to require solving an OT problem for any given sample of the outer expectation.
 
\looseness=-1
In practice, though, both $\thetagain(d)$ and $\ygain(d)$, can be approximated %
in a way that only requires us to solve only a single OT problem.
Leveraging recent work on conditional optimal transport \citep{carlier2016vector, kerrigan2024dynamic, chemseddine2024conditional, baptista2024conditional}, we show in the following result that when the cost function is given by a norm, these quantities are recovered as a limiting special case of $\jointgain(d)$ under a particular choice of $c$, the proof for which is given in Appendix~\ref{sect:appendix_mtd_details}. 
\begin{theorem}
    \label{thm:epsilon_cost_limit}
    For a fixed $1 \leq p < \infty$, consider the cost
    \begin{equation} \label{eqn:cot_cost1}
        c_\eta(\theta, y, \theta', y') = \eta|\theta - \theta'|^p_\Theta + |y - y'|^p_\YY
    \end{equation}
    for $\eta > 0$. Then, $\eta^{-1} \TT_{c_\eta}(d) \to \TT_{c_\Theta}^{{(\theta)}}(d)$ as $\eta \to 0^{+}$, where $c_\Theta(\theta, \theta') = |\theta - \theta'|^p_\Theta$. %
    Similarly, for
    \begin{equation} \label{eqn:cot_cost2}
        c_\psi(\theta, y, \theta', y') = |\theta - \theta'|^p_\Theta + \psi|y - y'|^p_\YY
    \end{equation}
    we have $\psi^{-1}\TT_{c_\psi}(d) \to \TT_{c_\YY}^{(y)}(d)$ as $\psi \to 0^{+}$ where $c_\YY(y, y') = |y - y'|^p_\YY$.    
\end{theorem}

\looseness=-1
As the \gls{mti} can be defined using any choice of cost function, we can thus think of it 
as a more general objective which subsumes the \gls{etti} and \gls{edti} as limiting cases in the choice of cost function. Thus, in practice, we can implement these alternatives simply by using a weighted cost function ($c_\eta$ or $c_\psi$) in the \gls{mti} with a small $\eta> 0$ or $\psi > 0$
for \gls{etti} and \gls{edti} respectively.
We, therefore, focus our subsequent discussion on the \gls{mti}.

We emphasize that while we have introduced various notions of transport dependence in terms of designing a single experiment, all constructions readily generalize to the sequential setting by updating the prior as discussed in \cref{sect:background}. While our experiments will focus on this sequential case, our notions of transport dependence can also be extended to the policy setting \citep{foster2021deep, ivanova2021implicit} by considering optimal transport over the entire experimental rollout.

\subsection{Estimation and Optimization}
\label{subsect:est_opt}

\looseness=-1
To use the \gls{mti} in practice, we must be able to efficiently estimate and optimize $\jointgain(d)$. Here, we briefly discuss the computational tools used later in Section \ref{sect:experiments}, but note that other algorithmic approaches to leverage transport dependence for experimental design may be viable.

We emphasize that our estimators for the \gls{mti} are purely sample-based. Assuming that we can draw samples $\theta \sim p(\theta)$ and $y | \theta \sim p(y \mid \theta, d)$, we can estimate the $\jointgain(d)$ without ever evaluating densities. %
By contrast, estimators for the \gls{mi} either require direct access to the likelihood density, or rely on learning approximations for density ratios
\citep{kleinegesse2020bayesian,kleinegesse2021gradient,ivanova2021implicit}, 
introducing substantial extra modelling and optimization complexity.

\paragraph{Estimation.} For discrete measures, the optimal transport problem becomes a linear program (LP) for which a wide range of numerical solvers have been proposed \citep{bonneel2011displacement, peyre2019computational}. When the distributions $p(\theta, y \mid d)$ and $p(\theta) p(y\mid d)$ are known and discrete, the OT problem may be solved directly using these distributions.

\looseness=-1
In practice, however, our distributions are often continuous or only accessible through sampling. In such cases, we approximate them using empirical measures based on samples. That is, we sample $(\theta_j, y_j) \iid p(\theta, y \mid d)$ and $(\theta_k', y_k') \iid p(\theta)p(y \mid d)$, followed by the approximations
\begin{equation}
    p(\theta, y \mid d) \approx \frac{1}{n} \sum_{j=1}^n \delta_{(\theta_j, y_j)},~ p(\theta) p(y \mid d) \approx \frac{1}{n} \sum_{k=1}^n \delta_{(\theta_k', y_k')}.
\end{equation}
This procedure yields the plug-in estimator \citep{boissard2014mean, fournier2015rate}
\begin{equation} \label{eqn:plug-in}
    \widehat{\jointgain}(d) = {\sf OT}_c\left[ \frac{1}{n} \sum_{j=1}^n \delta_{(\theta_j, y_j)}, \frac{1}{n} \sum_{k=1}^n \delta_{(\theta_k', y_k')} \right]
\end{equation}

which is asymptotically consistent \citep{dudley1969speed} as $n\to \infty$ and concentrates around its mean at an exponential rate \citep{bolley2007quantitative, boissard2011simple, weed2019sharp}, albeit with a positive bias that decreases with $n$ \citep{papp2025centered}.

\paragraph{Optimization.} Optimizing $\jointgain(d)$ poses a further challenge, assuming we cannot simply enumerate over possible designs.
We therefore now show how $\jointgain(d)$ can be optimized using stochastic gradient ascent provided the design space is continuous.

\looseness=-1
The only additional assumption needed for this is the existence of a differentiable reparameterization of $y$ with respect to $d$.
Namely, using the noise outsourcing lemma~\citep{kallenberg1997foundations}, then for any fixed noise distribution with appropriate reference measure, $q(\eta)$, there exists (subject to extremely weak assumptions) a function $h$ such that $y = h(\eta; \theta, d)$ for $\eta \sim q(\eta)$. If we further assume that $d \mapsto h(\eta; \theta, d)$ is differentiable for our chosen $q(\eta)$, then the LP approach above enables the calculation of $\nabla_d \widehat{\jointgain}(d)$ via automatic differentiation. Hence, we may perform gradient-based design optimization in this setting.
We refer to \citet[Chapter~9]{peyre2019computational} for a further discussion of the differentiability of optimal transport discrepancies. 

Note that this differentiable reparameterisation assumption is the same as in implicit MI methods and is often satisfied even when evaluating $p(y \mid \theta, d)$ is itself intractable: many, if not most, intractable likelihood models are based on stochastic simulators, with the intractability coming from deterministic mappings of stochastic variables or stochastic differential equations~\citep{cranmer2020frontier}.

\section{Comparing Mutual Information and Mutual Transport Dependence}
\label{sect:theory}

In this section we theoretically analyse how our proposed transport-based criteria relate to the classical expected information gain (\gls{mi}). Although the \gls{mti} and \gls{mi} originate from different principles, there are some interesting links between the two as we now show. In particular, under quadratic costs, the transport dependencies can be upper-bounded by the MI.

\begin{theorem} \label{thm:eig_bounds_mtd}
    Suppose the prior $p(\theta)$ is strictly log-concave, i.e., there exists some $\lambda_\theta > 0$ with $-\nabla^2 \log p(\theta) \succeq \lambda_\theta I$. For $c(\theta, \theta') = |\theta - \theta'|^2$, we have
    \begin{equation}
        \lambda_\theta \thetagain(d) \leq 2  \EIG(d).
    \end{equation}
    Similarly, if the marginal $p(y \mid d)$ is strictly log-concave with parameter $\lambda_{y \mid d}$, and $c(y, y') = |y - y'|^2$, then
    \begin{equation}
        \lambda_{y \mid d} \ygain(d) \leq 2 \EIG(d).
    \end{equation}
    When both the prior and likelihood satisfy these assumptions, under cost $c(\theta, \theta', y, y') = \eta |\theta - \theta'|^2 + |y - y'|^2$, for $\lambda = \max\{\lambda_\theta / \eta, \lambda_{y \mid d} \}$ we have
    \begin{equation}
        \lambda \jointgain(d) \leq 2 \EIG(d).
    \end{equation}
\end{theorem}

See \cref{sect:appendix_theory} for a proof and extended discussion.

\looseness=-1
We note that Theorem~\ref{thm:eig_bounds_mtd} holds for quadratic costs on any space, and in particular remains valid under any transformations of $\Theta$ and $\YY$ when the quadratic cost is computed in the new coordinates. Thus, if the \gls{mti} is large under \textit{any} such transformation, the MI must necessarily also be large. This suggests a robustness to misspecification in the selected cost function in the \gls{mti}, as selecting for designs under a given particular cost 
must ensure a minimum level in the MI as well.

\looseness=-1
As a further point of comparison, we derive a closed-form expression for the \gls{mti} for a simple linear-Gaussian model under quadratic costs. %
We allow for the possibility of the observation noise $\sigma^2_{d,\theta}$ to vary with design to demonstrate how the  value of the MI diverges as $\sigma^2_{d,\theta} \to 0$, 
i.e., the likelihood approaches a deterministic outcome. 
On the other hand, $\jointgain(d)$ remains bounded for all designs $d$ and all noise $\sigma^2_{d,\theta}$. This boundedness makes the \gls{mti} a quantitative and stable measure of dependence even in scenarios approaching determinism.

\begin{theorem}
    Suppose $\theta \in \Theta = \R^n$ has a standard normal prior $p(\theta) = \NN(0, I_n)$, designs are vectors $d \in \DD = \R^n$, and $y \in \YY = \R$ has likelihood $p(y \mid \theta, d) = \NN\left( \langle d, \theta \rangle, \sigma^2_{d,\theta} \right)$. Under the quadratic cost, we have
    \begin{align}
    \jointgain(d) = 2 \bigg(&1 + \sigma^2_{d,\theta} + |d|^2 \label{eqn:mti_linear_gaussian} \\
    &- \sqrt{1 + (|d|^2 + \sigma^2_{d,\theta})^2 + 2\sqrt{|d|^2 + \sigma^2_{d,\theta}}} \bigg). \nonumber
    \end{align}
    Moreover, $\jointgain(d) \leq 2$. 
    On the other hand, the MI is
    \begin{equation}
        \EIG(d) = \frac{1}{2}\log\left( 1 + |d|^2/\sigma^2_{d,\theta}\right)
    \end{equation}
    which is unbounded as $\sigma^2_{d,\theta}\to 0$.
\end{theorem}

\begin{table*}[ht]
\centering
\caption{Metrics for the CES model after $T=10$ design iterations, averaged over 50 random seeds ($\pm$ one standard error). Designs produced by the MTD yield lower RMSEs than PCE on average for all parameters.}
\vspace{0.5em}
\label{tab:ces_50seeds}
\begin{tabular}{lccc|ccc}
\hline
 & $\rho$ & $\alpha$ & $u$ & $\sigma$ & $\beta$ & $\tau$ \\
\hline
Random & $0.251 {\scriptstyle \pm 0.025}$ & $0.116 {\scriptstyle \pm 0.016}$ & $36.365 {\scriptstyle \pm 7.341}$ & $317.219 {\scriptstyle \pm 81.866}$ & $0.727 {\scriptstyle \pm 0.088}$ & $0.740 {\scriptstyle \pm 0.106}$ \\
PCE & $0.047 {\scriptstyle \pm 0.012}$ & $0.036 {\scriptstyle \pm 0.013}$ & $8.902 {\scriptstyle \pm 5.749}$ & ${24.942} {\scriptstyle \pm 15.697}$ & $0.201 {\scriptstyle \pm 0.060}$ & $0.100 {\scriptstyle \pm 0.048}$ \\
MTD & ${0.018} {\scriptstyle \pm 0.005}$ & ${0.009} {\scriptstyle \pm 0.001}$ & ${3.671} {\scriptstyle \pm 2.810}$ & ${1.767} {\scriptstyle \pm 1.009}$ & ${0.058}  {\scriptstyle \pm 0.008}$ & ${0.049} {\scriptstyle \pm 0.012}$ \\
MTD ($\TT_{c_\dagger}$) & ${0.022} {\scriptstyle \pm 0.008}$ & ${0.012} {\scriptstyle \pm 0.004}$ & $10.941 {\scriptstyle \pm 10.107}$ & ${0.534} {\scriptstyle \pm 0.195}$ & ${0.069} {\scriptstyle \pm 0.013}$ & ${0.053} {\scriptstyle \pm 0.013}$ \\
\hline
\end{tabular}
\end{table*}

\section{Related Work}
\label{sect:related_work}

Classical optimal experimental design criteria trace back to frequentist approaches based on the Fisher information matrix \citep{fisher1935design, wald1943efficient, kiefer1959optimum, kiefer1974general, pukelsheim2006optimal}. While powerful in some settings, these methods often rely on asymptotic approximations and are limited when models are nonlinear \citep{ryan2016review, rainforth2024modern}. As they depend only on local (second-order) information about $\theta$, they can lose fidelity compared to criteria based on the full joint distribution $p(\theta, y \mid d)$.

\looseness=-1
Bayesian experimental design (BED) addresses many of these limitations by evaluating designs using objectives that can be viewed as measuring expected reduction in uncertainty on $\theta$~\citep{degroot1962uncertainty,dawid1998coherent,smith2025rethinking}.
Here this uncertainty is typically measured using (differential) entropy to produce the MI or expected information gain~\citep{lindley1956measure}, especially in the contemporary literature~\citep{huan2014gradient, foster2021thesis, foster2021deep, ao2024estimating, iollo2024pasoa}. 
The trace or determinant of the posterior covariance matrix have also occasionally be used instead~\citep{vanlier2012bayesian,ryan2016review,huan2024optimal}, but this requires expensive nested inference procedures to be performed that are typically even more costly than MI optimization.

\looseness=-1
Concurrent work by \citet{helin2025bayesian} also studies the \acrlong{etti} $\thetagain(d)$ under the specific choice $c(\theta, \theta') = |\theta - \theta'|^p$, $p \in [1, \infty)$, as an objective for experimental design. 
In light of Theorem \ref{thm:epsilon_cost_limit}, this can be seen as a limiting case of our more general \gls{mti} criterion under a Euclidean cost assumption. 
Their work is primarily theoretical with no quantitative comparisons against the \gls{mi}, and they do not propose a practical method for optimizing the \gls{etti} and in particular overcoming its double intractability.
Our work, by contrast, develops a sample-based, differentiable framework applicable for general costs and empirically demonstrates its efficacy for sequential OED.

\looseness=-1
Optimal-transport based notions of statistical dependency have also been considered in areas such as representation learning \citep{ozair2019wasserstein} independence testing \citep{warren2021wasserstein, wiesel2021measuring, nies2025transport}, and fairness \citep{leteno2023fair}. These works, however, are not concerned with experimental design and also focus exclusively on Euclidean costs. Our work adds to this growing literature of geometric dependency measures by introducing the \acrlong{mti} for OED under general cost functions.

\section{Experiments}
\label{sect:experiments}

We now evaluate the proposed methodology on both standard benchmark experimental design tasks and variations on these that have particular error desiderata in our final estimates. In each setting, we use the \gls{mti} as the design criterion, sequentially selecting experiments by optimizing $\jointgain(d)$ as explained in~\Cref{subsect:est_opt}. After each design is chosen, we perform posterior inference over $\theta$ and proceed to the next experimental iteration. All results are reported over either 25 or 50 random seeds, where each random seed constitutes a different ground-truth value of $\theta$. We compare against the MI throughout, using PCE \citep{foster2019variational} as a well-known estimator for this quantity. 
We note that unlike our MTD approach, PCE is an explicit estimator that requires direct access to the likelihood density, thereby providing a stronger baseline than more directly comparable, but also more complex, implicit MI approaches.
See \cref{sect:appendix_experiment_details} for details.\footnote{Experiment code: \href{https://github.com/GavinKerrigan/mtd}{github.com/GavinKerrigan/mtd}}

\paragraph{CES.} The first problem we consider is Constant Elasticity of Substitution (CES) \citep{arrow1961capital, foster2020unified, blau2022optimizing, iollo2024pasoa, hedman2025step}, arising from behavioral economics. In this problem, a participant compares two baskets $d_1, d_2 \in [0, 100]^3$ consisting of various amounts of three different goods. Given two baskets, the participant provides a scalar response $y \in [0, 1]$ indicating their subjective preference between the baskets. The design variable $d = (d_1, d_2)$ is thus six-dimensional, and the goal is to recover the latent parameters $\theta = (\rho, \alpha_1, \alpha_2, \alpha_3, u) \in \R^5$ governing the participant's preferences. This is a particularly challenging design problem, as large regions of the design space result in uninformative outcomes $y \in \{0, 1 \}$. We sequentially design $T=10$ experiment iterations.

\paragraph{Source Location Finding.} Our second problem is source location finding (LF) \citep{sheng2005maximum, foster2021deep, ivanova2021implicit, blau2022optimizing, iollo2024bayesian}. In this, our goal is to estimate the spatial location of two sources $\theta_1, \theta_2 \in \R^2$. Each source emits a signal which decays according to an inverse square law. At each experiment iteration, a sensor is placed at a location $d \in \R^2$ which records a noisy measurement $y \in \R$ of the total signal intensity at the sensor location. Here, we design $T=25$ experiments.

\subsection{MTD under Euclidean Costs}

While one of the main appeals of the \gls{mti} is that it allows for flexible cost functions, in this section we first consider the quadratic cost $c(\theta, y, \theta', y') = |\theta - \theta'|^2 + |y - y'|^2$ as a reasonable default choice. Our first set of experiments demonstrates that even under this default setting, \gls{mti}-optimal designs can exceed the performance of MI-optimal designs in terms of recovering an unknown parameter. 

In \cref{tab:ces_50seeds}, we evaluate the \gls{mti} on the CES problem in terms of the final RMSE between posterior samples after $T=10$ experiment iterations and the true value of $\theta$. Designs produced by optimizing \gls{mti} achieve lower RMSEs than those produced by optimizing \gls{mi}. %

Similarly, in \cref{fig:lf_rmse}, we plot the RMSE between posterior samples and the true $\theta$ value on the LF problem over the course of $T=25$ design iterations. We observe that the \gls{mti} yields lower RMSEs throughout most of the iterations, but designs produced by optimizing the MI yield similar RMSEs at the final iteration. 

For the LF problem, both \gls{mi} and \gls{mti} are optimized using five random restarts, i.e., we generate five candidate designs and retain the best under the given objective. This approach serves not only to mitigate sensitivity to initialization but also to systematically improve design quality. In particular, for later iterations of the LF problem, where the posterior over $\theta$ becomes highly concentrated, the restart strategy provides a simple yet effective mechanism to ensure robustness against poor initializations.

\looseness=-1
In terms of runtime, for either problem optimizing a single design under \gls{mti} requires approximately ${30}$ seconds of wall-clock time, whereas optimizing the same design with PCE takes roughly ${120}$ seconds, with both objectives run to convergence. While these runtimes are sensitive to implementation choices and could likely be reduced through more careful tuning or normalization of compute budgets, the key observation is that \gls{mti} is comparably fast to previous approaches and potentially faster.

\subsection{MTD under Transformations}

\begin{figure}
    \centering
    \includegraphics
    {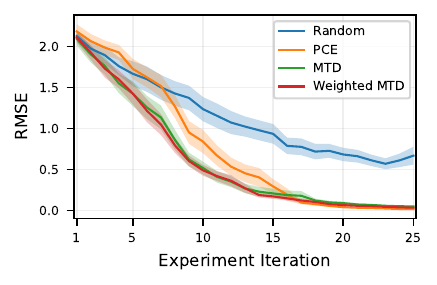}
    \caption{RMSE between posterior $\theta$ samples and ground-truth on the location finding problem, averaged over 25 seeds ($\pm$ one standard error).}
    \label{fig:lf_rmse}
\end{figure}

\looseness=-1
While we may explicitly specify a cost for \gls{mti}, costs can also be defined \textit{implicitly} through  transformations of the underlying sample spaces. Concretely, if $f: \Theta \to \Theta^\dagger, g: \YY \to \YY^\dagger$ are two transformations, we may define $c_{\dagger}(\theta, \theta', y, y') = |f(\theta) - f(\theta')|^2 + |g(y) - g(y')|^2,$
i.e., quadratic cost in transformed coordinates. This is useful when we wish to measure errors in a particular space.
Note that the MI between $f(\theta)$ and $g(y)$ is equal to that between $\theta$ and $y$ if $f$ and $g$ are injective, prohibiting the same trick from being meaningfully employed.

We illustrate this on CES using the transformations
\begin{equation}
    \sigma = 1 / (1 - \rho) \qquad \beta_i = \log(\alpha_i)/g(\alpha) \qquad \tau = \log(u)
\end{equation}
where $g(\alpha)$ is the geometric mean of $\alpha$. These transformations are interpretable: $\sigma$ is the elasticity \citep{arrow1961capital}, $\beta$ the centered-log-ratio of $\alpha$, capturing relative importance of goods, and $\tau=\log u$ a natural reparametrization under its lognormal prior.

To evaluate this approach, we generate designs using $\TT_{c_\dagger}(d)$ in the transformed variables, implicitly altering the cost. \cref{tab:ces_50seeds} reports posterior RMSEs for PCE, \gls{mti} on the original scale $\jointgain(d)$, and \gls{mti} with the transformed cost $\TT_{c_\dagger}(d)$. On the original parameters, $\TT_{c_\dagger}(d)$ performs comparably to $\jointgain(d)$ for $\rho$ and $\alpha$ but somewhat worse for $u$, suggesting the untransformed version remains preferable when evaluation is performed directly on the original parameters.

On the transformed scale, $\TT_{c_\dagger}(d)$ and $\jointgain(d)$ perform similarly for $\beta$ and $\tau$. However, there are clearer differences in $\sigma$. In particular, we see that PCE exhibits high RMSE in $\sigma$. This is because PCE occasionally yields poor designs which are unable to identify that $\rho \neq 1$, leading to high errors in $\sigma = (1 - \rho)^{-1}$. $\jointgain(d)$ generally yields higher quality designs which reliably identify $\rho$ and thus obtain low errors in $\sigma$. The transformed $\TT_{c_\dagger}(d)$, though, implicitly upweights designs where $\sigma$ is large, leading to low RMSE values. Notably, there is no natural analogue of changing PCE to target  RMSE in $\sigma$ directly. Overall, this provides evidence that the \gls{mti} can be tailored to particular downstream metrics.

\subsection{Weighted Cost Functions}
\label{sect:experiments_weighting}

\begin{figure}[t]
    \centering
    \includegraphics[]{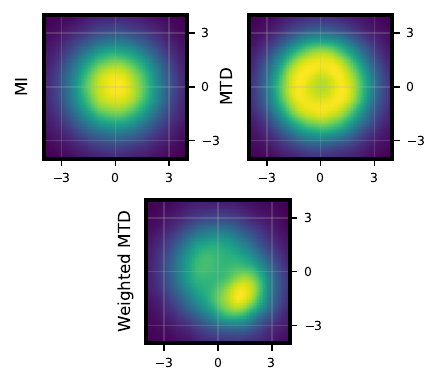}
    \caption{Estimated values of the EIG (via PCE) and the \gls{mti} for the 2D location-finding model with two sources. \textit{Top row}: the EIG is maximized at the origin, whereas the MTD favors off-center designs, illustrating its symmetry-breaking behavior. \textit{Bottom row}: with additional weighting of the cost function, the \gls{mti} can be tuned to favor specific regions of the design space.}
    \label{fig:2d_lf_weighting}
\end{figure}

We next evaluate the \gls{mti} on a variation of the LF problem which highlights its ability to incorporate downstream objectives through an appropriate choice of cost function. Here, the goal is not only to localize the sources, but also to rapidly determine whether a source lies in a critical region $\RR \subset \Theta$.%

To capture this preference, we define a weighted cost $c_w(\theta, y, \theta', y') = w(\theta) \left( |\theta - \theta'|^2 + |y - y'|^2 \right)$
where $(\theta, y) \sim p(\theta, y \mid d)$ and $(\theta', y') \sim p(\theta') p(y' \mid d)$ and the weight is given by $w(\theta) = b + \sum_{k=1}^2g(\theta_k- \mu),$
with $b>0$ a bias and $g$ a bump function supported on $\RR$, a ball of radius $1.5$ centered at $\mu=(1.5, -1.5)$. See \cref{sect:appendix_experiment_details} for details. Intuitively, the cost is up-weighted whenever the ``true'' $\theta$ has a source in $\RR$. In such cases, the \gls{mti} $\TT_{c_w}(d)$ increases, thus yielding designs that prioritize detecting whether a source is present in $\RR$. We stress that this represents only one possible weighting scheme, and other choices could be used to encode different downstream preferences.

\begin{figure}[t]
    \centering
    \includegraphics{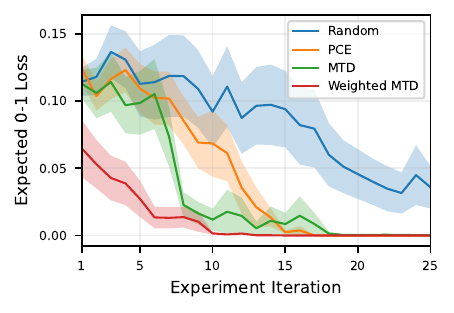}
    \caption{Expected zero-one loss for the 2D location-finding model. The weighted \gls{mti} rapidly determines whether $\theta \in \RR$.}
    \label{fig:2d_lf_weighting_quant}
\end{figure}

\looseness=-1
In \cref{fig:2d_lf_weighting}, we plot $\TT_{c_w}(d)$ under the weighted cost for the LF task. As intended, the objective is upweighted in $\RR$, encouraging designs to be placed in this region. We also note that the unweighted \gls{mti} (with the quadratic cost) exhibits a symmetry breaking behavior, whereas the MI favors designs at the origin, thereby preserving symmetry after posterior updates (as in \cref{fig:leading}).

\looseness=-1
We then design a sequence of $T=25$ experiments using the weighted cost $c_w$. In \cref{fig:2d_lf_weighting_quant}, we plot the mean zero-one loss of $\theta \in \RR$, i.e., $\E_{p(\theta \mid h_t)}[ | \mathbf{1}[\theta_{\textrm{true}} \in \RR] - \mathbf{1}[\theta \in \RR]  |  ]$
which directly measures if we have detected $\theta \in \RR$. While PCE and unweighted \gls{mti} eventually resolve this uncertainty, the weighted \gls{mti} achieves a much faster reduction. We emphasize that the MI cannot be easily adapted to the task of first identifying if $\theta \in \RR$ before exploring the rest of the space, %
demonstrating our framework’s flexibility to encode task-specific preferences via the cost.
Further, \cref{fig:lf_rmse} shows the RMSE under $\TT_{c_w}$ matches that of $\jointgain(d)$, confirming that we have not sacrificed performance in terms of identifying $\theta$ for this auxiliary objective.

\section{Conclusion}

We introduce the \acrfull{mti}, a novel class of geometric criteria for optimal experimental design. By quantifying the value of an experiment through an optimal transport divergence with an explicit sample-level cost, the \gls{mti} allows us to encode domain knowledge and task-specific objectives directly into the design criterion. We show that optimizing the \gls{mti} produces highly effective designs on standard benchmarks, and that tailoring the cost function enables alignment with particular experimental goals. Overall, the \gls{mti} offers a flexible, geometry-aware objective for \gls{oed}, providing a practical tool for designing experiments that reflect both statistical dependence and the experimenter’s real-world priorities.

\subsubsection*{Acknowledgements}
GK and TR are supported by the UK EPSRC grant EP/Y037200/1.

\bibliography{refs}

\clearpage
\appendix

\include{supplement}

\end{document}

%% file: supplement.tex
\onecolumn
\aistatstitle{A Geometric Approach to Optimal Experimental Design: \\
Supplementary Materials}

\section{Transport Dependencies}
\label{sect:appendix_mtd_details}

In this section, we provide a more formal discussion of the \gls{mti}. We work under standard assumptions throughout, which are sufficient for guaranteeing that a solution to the OT problem exists.
\begin{assumption} \label{assump:polish}
    The spaces $\Theta$ and $\YY$ are Polish spaces.
\end{assumption}
\begin{assumption} \label{assump:lsc}
    All cost functions are lower-semicontinuous and non-negative.
\end{assumption}

Under Assumptions~\ref{assump:polish}-\ref{assump:lsc}, minimizers to the OT problem in \cref{eqn:ot_cost} defined on either $\Theta$ or $\YY$ are guaranteed to exist \citep[Theorem~1.5]{ambrosio2012user}. Similarly, if $\Theta, \YY$ are Polish, then $\Theta \times \YY$ is Polish when equipped with the product topology, so that minimizers to an OT problem on the product spaces also exist under Assumptions~\ref{assump:polish}-\ref{assump:lsc}. Additional assumptions on $c$ are necessary, though, to guarantee that $\jointgain(d)$ (and the \gls{edti}/\gls{etti}) is finite. A trivially sufficient condition is that $c$ is bounded from above. Other sufficient conditions can be given under moment assumptions of the corresponding densities. See \cref{lemma:mtd_lessthan} and \cref{thm:moment_bounds} below.

As discussed in the main paper, the \gls{etti} and \gls{edti} are closely related to notions arising from conditional optimal transport \citep{carlier2016vector, hosseini2025conditional, kerrigan2024dynamic, chemseddine2024conditional, baptista2024conditional}. Viewing our transport dependencies from this lens is a fruitful avenue for theoretical analysis. We begin by recalling the notion of a triangular coupling \citep{hosseini2025conditional}, which gives a notion of couplings that fix certain variables. 

\begin{definition}[Triangular Couplings]
    A coupling $\gamma \in \Pi(p(\theta, y \mid d), p(\theta)p(y \mid d)$ is said to be $\YY$-triangular if draws $(\theta, y, \theta' ,y') \sim \gamma$ are such that $y = y'$ almost surely. Similarly, $\gamma$ is said to be $\Theta$-triangular if draws $(\theta, y, \theta', y') \sim \gamma$ are such that $\theta = \theta'$ almost surely. 
\end{definition}

For the sake of brevity, we will write $\Pi := \Pi(p(\theta, y \mid d), p(\theta)p(y \mid d)$ for the set of all couplings, $\Pi_\YY := \Pi_\YY(p(\theta, y \mid d), p(\theta)p(y \mid d))$ for the set of $\YY$-triangular couplings, and $\Pi_\Theta := \Pi_\Theta(p(\theta, y \mid d), p(\theta)p(y \mid d))$ for the set of $\Theta$-triangular couplings.

We begin with a lemma which allows us to bound the \gls{mti} by the \gls{edti} and \gls{etti}. This result shows that if the \gls{mti} is large, then both corresponding transport divergences on $\Theta$ or $\YY$ must also be large. This is particularly interpretable in terms of the \gls{etti} $\thetagain(d)$, where we see that large \gls{mti} implies that there is a large transport divergence between the posterior and prior, on average across the marginal $p(y \mid d)$.

\begin{lemma} \label{lemma:mtd_lessthan}
    Fix $p \in [1, \infty)$. Suppose $\Theta, \YY$ are separable Hilbert spaces. Consider the cost function $c(\theta, y, \theta', y') = \eta|\theta - \theta'|^p + |y - y'|^p$ for a given $\eta > 0$. Write $c_\Theta(\theta, \theta') = |\theta - \theta'|^p$ and $c_\YY(y, y') = |y - y'|^p$. Assume that $p(\theta, y \mid d)$ has finite $p$th moment for a given $d \in \DD$. Then,
    \begin{equation}
        \jointgain(d) \leq \eta\TT_{c_\Theta}^{(\theta)}(d) \qquad \text{ and } \qquad \jointgain(d) \leq \TT_{c_\YY}^{(y)}(d).
    \end{equation}

    Furthermore, both $\TT_{c_\Theta}^{(\theta)}(d)$ and $\TT_{c_\YY}^{(y)}(d)$ are finite.
\end{lemma}

\begin{proof}
    First consider the $\eta = 1$ case. Observe that $p(\theta, y \mid d)$ having finite $p$th moments immediately implies that both $p(\theta)$ and $p(y \mid d)$ also have finite $p$th moments. Note further that $p(\theta, y \mid d)$ and $p(\theta)p(y \mid d)$ have the same marginals in both $\Theta$ and $\YY$ space. It follows that both $\TT_{c_\Theta}^{(\theta)}(d)$ and $\TT_{c_\YY}^{(y)}$ are $p$th powers of conditional Wasserstein metrics \citep[Definition~2]{kerrigan2024dynamic}. By \citet[Prop.~2]{kerrigan2024dynamic}, both $\TT_{c_\Theta}^{(\theta)}(d)$ and $\TT_{c_\YY}^{(y)}(d)$ are finite, and furthermore conditional Wasserstein metrics upper bound the Wasserstein metric on the corresponding joint measure. This proves the claim for $\eta=1$.
    
    For the general $\eta > 0$ case, observe that we can equip $\Theta$ with the alternative inner product $\langle \theta, \theta'\rangle_\eta = \eta^{1/p} \langle \theta, \theta'\rangle_\Theta$ which yields the norm $|\theta|_{\eta} = \eta^{1/p}|\theta|_\Theta$. Since the preceding argument applies to general separable Hilbert spaces, it also holds when $\Theta$ is equipped with the alternative norm, yielding
    \begin{equation}
        \jointgain(d) \leq \TT_{c_{\eta, \Theta}}^{(\theta)}(d)
    \end{equation}
    for $c_{\Theta, \eta}(\theta, \theta') = \eta|\theta - \theta'|^p$. Further, observe that
    \begin{align}
        \TT_{c_\Theta, \eta}^{(\theta)}(d) &= \int_\YY \left[ \min_{\gamma \in \Pi(p(\theta), p(\theta \mid y, d))} \int_{\Theta^2} \eta |\theta - \theta'|^2 \d \gamma(\theta, \theta') \right] p(y \mid d) \d y \\
        &= \eta \int_\YY \left[ \min_{\gamma \in \Pi(p(\theta), p(\theta \mid y, d))} \int_{\Theta^2} |\theta - \theta'|^2 \d \gamma(\theta, \theta') \right] p(y \mid d) \d y \\
        &= \eta \TT_{c_\Theta}^{(\theta)}(d).
    \end{align}
    This yields the desired claim.    
\end{proof}

\subsection{Moment Bounds}

In this section, we prove several upper bounds on our transport dependencies which rely on moments of the underlying distributions. In particular, this theorem shows that when there is a link between the \gls{edti} $\ygain(d)$ (and thus also the \gls{mti} by \cref{lemma:mtd_lessthan}) and the predictive variance of $p(y \mid d)$. Intuitively, this means that we should expect that maximizing the \gls{edti} (and \gls{mti}) should select for designs for which there is a high amount of variance in the experimental outcome.

\begin{theorem} \label{thm:moment_bounds}
    Suppose $\Theta, \YY$ are separable Hilbert spaces. Fix $p \in [1, \infty)$ and suppose that $p(\theta, y \mid d)$ has finite $p$th moment. For $c_\Theta(\theta, \theta') = |\theta - \theta|^p$, we have
    \begin{equation}
        \TT_{c_\Theta}^{(\theta)}(d) \leq 2^p \E_{p(\theta)}|\theta - \E [\theta]|^p.
    \end{equation}
    Similarly, for $c_\YY(y, y') = |y - y'|^p$,
    \begin{equation}
        \TT_{c_\YY}^{(y)}(d) \leq 2^p \E_{p(y \mid d)}|y - \E [y \mid d]|^p.
    \end{equation}
\end{theorem}

\begin{proof}
    We begin with the bound on $\thetagain(d)$. Observe that $\gamma(\theta, y, \theta', y') = p(\theta, y \mid \xi) p(\theta') \delta[y' = y]$ is a valid triangular coupling of $p(\theta, y \mid d)$ and $p(\theta')p(y' \mid d)$. By the convexity of $x \mapsto x^p$, we have
    \begin{align}
        \TT_{c_\Theta}^{(\theta)}(d) &\leq \E_{\gamma}|\theta - \theta'|^p = \E_{\gamma}| (\theta - \E_{p(\theta)} [\theta]) - (\theta' - \E_{p(\theta)}[\theta])|^p \\
        &\leq 2^{p-1} \E_\gamma \left[ |\theta - \E\theta|^p + |\theta' - \E\theta|^p \right] \\
        &= 2^{p-1} \left( \int_{\Theta^2\times \YY} |\theta - \E \theta|^p \d p(\theta')\d p(\theta, y \mid d) + \int_{\Theta^2\times \YY} |\theta' - \E \theta|^p \d p(\theta')\d p(\theta, y \mid d) \right) \\
        &= 2^p \E_{p(\theta)} |\theta - \E \theta|^p.
    \end{align}
    where the last line follows by marginalization. The proof for $\TT_{c_\YY}^{(y)}(d)$ is analogous. %
\end{proof}

We note that in the case $p=2$, one may use the identity \begin{align}
        \E_{\gamma}|\theta - \theta'|^2 &= \E_{\gamma}| (\theta - \E \theta) - (\theta' - \E \theta)|^2 \\
        &= 2\E_{p(\theta)} |\theta - \E \theta|^2 - 2 \E_{\gamma} \langle \theta - \E\theta, \theta' - \E\theta \rangle
\end{align}
rather than convexity to obtain a sharper constant.

\subsection{Limiting Cases of the MTD}

In this section, we prove that the \gls{etti} and \gls{edti} can be obtained as limiting cases of the \gls{mti} under a particular choice of cost. We refer to \cref{subsect:alternatives} for a discussion of this result and its implications for OED. The following is a formal restatement of \cref{thm:epsilon_cost_limit}.

\begin{theorem}
    \label{thm:epsilon_cost_limit2}
    Suppose $\Theta, \YY$ are separable Hilbert spaces. Fix $p \in [1, \infty)$ and assume that $p(\theta, y \mid d)$ has finite $p$th moment. Consider the cost
    \begin{equation}
        c_\eta(\theta, y, \theta', y') = \eta|\theta - \theta'|^p_\Theta + |y - y'|^p_\YY
    \end{equation}
    for $\eta > 0$. Then, $\eta^{-1} \TT_{c_\eta}(d) \to \TT_{c_\Theta}^{{(\theta)}}(d)$ as $\eta \to 0^{+}$, where $c_\Theta(\theta, \theta') = |\theta - \theta'|^p_\Theta$. %
    Similarly, for
    \begin{equation}
        c_\psi(\theta, y, \theta', y') = |\theta - \theta'|^p_\Theta + \psi|y - y'|^p_\YY
    \end{equation}
    we have $\psi^{-1}\TT_{c_\psi}(d) \to \TT_{c_\YY}^{(y)}(d)$ as $\psi \to 0^{+}$ where $c_\YY(y, y') = |y - y'|^p_\YY$.    
\end{theorem}
\begin{proof}
    We begin with the first claim. Let $\gamma^* \in \Pi_\YY$ be an optimal $\YY$-triangular coupling for the cost $c_\Theta(\theta, \theta') = |\theta - \theta'|^p$ and let $\gamma_\eta \in \Pi$ be an optimal coupling for the cost $c_\eta$. Such optimal couplings exist and yield finite costs due to our moment assumption and \cref{thm:moment_bounds}.
    
    Since $\Pi_\YY \subset \Pi$, using the $\YY$-triangularity of $\gamma^*$ we may upper bound the \gls{mti} under the cost $c_\eta$ by
    \begin{align}
        \TT_{c_\eta}(d) &= \int c_\eta \d \gamma_\eta \leq \int c_\eta \d \gamma^* \\
        &= \int \eta|\theta - \theta'|^p \d \gamma^* + \int |y - y'|^p \d \gamma^*  \\
        &= \eta\int |\theta - \theta'|^p \d \gamma^*.
    \end{align}

    Consequently, by expanding out the definition of $\TT_{c_\eta}(d)$ we see that
    \begin{equation} \label{eqn:squeeze}
        0 \leq \eta^{-1} \int |y - y'|^p \d \gamma_\eta \leq \int |\theta - \theta'|^p \d (\gamma^* - \gamma_\eta).
    \end{equation}
    This yields $ \int |\theta - \theta'|^p \d \gamma_\eta \leq \int |\theta - \theta'|^p \d \gamma^*$,
    so that $\limsup_{\eta \to 0^+} \int |\theta - \theta'|^p \d \gamma_\eta \leq \int |\theta - \theta'|^p \d \gamma^*$, which is finite as we assume $p(\theta, y \mid d)$ has finite $p$th moment. \citet[Prop.~3.11]{hosseini2025conditional} show that as $\eta \to 0^+$, we have $\gamma_\eta \to \gamma^*$ in the weak sense. By the Portmanteau theorem, this weak convergence implies $\liminf_{\eta \to 0^+} \int |\theta - \theta'|^p \d \gamma_\eta \geq \int |\theta - \theta'|^p \d \gamma^*$. We have thus shown
    \begin{equation} \label{eqn:lim1}
       \lim_{\eta \to 0^+} \int |\theta - \theta'|^p \d \gamma_\eta = \int |\theta - \theta'|^p \d \gamma^*.
    \end{equation}
    By \cref{eqn:squeeze}, we thus also have
    \begin{equation} \label{eqn:lim2}
         \lim_{\eta \to 0^+} \eta^{-1} \int |y - y'|^p \d \gamma_\eta = 0.
    \end{equation}

    Together, \cref{eqn:lim1} and \cref{eqn:lim2} imply that
    \begin{align}
         \lim_{\eta \to 0^+} \eta^{-1} \TT_{c_\eta}(d) &=  \lim_{\eta \to 0^+} \left( \int |\theta - \theta'|^p \d \gamma_\eta + \eta^{-1}\int|y - y'|^p \d \gamma \right) \\
         &= \int |\theta - \theta'|^p \d \gamma^* = \TT_{c_\Theta}^{(\theta)}(d).
    \end{align}

    The second claim can be shown with a directly analogous argument, interchanging the roles of $\theta$ and $y$.
\end{proof}

\subsection{Estimation and Optimization}

\paragraph{Estimation.} 
Here we include further details regarding the estimation and optimization of $\jointgain(d)$. We focus on the setting where $\theta, y$ are continuous. In principle, to form our plug-in estimate $\widehat{\jointgain}(d)$ in \cref{eqn:plug-in}, we require samples
\begin{equation}
    (\theta_j, y_j) \iid p(\theta, y \mid d) \quad j=1, \dots, n \qquad (\theta_k', y_k') \iid p(\theta)p(y \mid d) \quad k=1, \dots, n.
\end{equation}

The product of marginals $p(\theta)p(y \mid d)$ can be sampled by drawing $\theta_k', \theta_k'' \sim p(\theta)$ followed by sampling $y_k' \sim p(y \mid \theta_k'', d)$. In principle, this requires $2n$ draws from the prior and simulations from the likelihood. However, in practice we reduce this to $n$ draws by first obtaining $(\theta_j, y_j) \iid p(\theta, y \mid d)$ followed by choosing a derangement $\sigma$ (i.e., a permutation with no fixed points) and defining $\theta_k' = \theta_j, y_k' = y_{\sigma(j)}$, breaking the dependency. This allows for a computational speedup (particularly when simulating the likelihood is expensive) and further can serve to reduce the bias of our estimator. 
We will write $\mu_n, \nu_n$ for these two empirical measures, i.e.,
\begin{equation}
    \mu_n = \frac{1}{n} \sum_{j=1}^n \delta_{(\theta_j, y_j)} \qquad \nu_n = \frac{1}{n} \sum_{k=1}^n \delta_{(\theta_k', y_k')}.
\end{equation}

This yields the plug-in estimator,
\begin{equation}
    \jointgain(d) \approx \widehat{\jointgain}(d) = {\sf OT}_c\left[\mu_n, \nu_n \right],
\end{equation}
which can be solved using efficient linear-programming techniques \citep{bonneel2011displacement, peyre2019computational, papp2025centered}. In particular, this requires forming the cost matrix $C(d) \in \R^{n \times n}$ with entries $C_{j,k}(d) = c(\theta_j, y_j, \theta_k', y_k')$. Note that $C(d)$ is a function of $d$ as $y_j, y_k'$ depend on $d$ through the likelihood. Further, the value of $\widehat{\jointgain}(d)$ is determined entirely by this cost matrix $C(d)$. We will write $\GG(C)$ for the minimal transport cost obtained for a given cost matrix $C$.

\paragraph{Optimization.}
We require an estimate of $\nabla_d \jointgain(d)$, which we obtain by computing the gradient $\nabla_d \widehat{\jointgain}(d)$ of our plug-in estimator. Key to computing this is the envelope theorem \citep{danskin1967theory, bertsekas1971control, bertsekas1997nonlinear}, which shows that we can obtain the gradient of $C \mapsto \GG(C)$ in terms of the optimal transport plan. To be more precise, this mapping is not differentiable, but we use $\partial \GG(C)$ to represent the superdifferential \citep{boyd2004convex} of $\GG$ at $C$, i.e., the set of all $v \in \R^{n \times n}$ satisfying
\begin{equation}
    \GG(C') - \GG(C) \leq \langle v, C' - C\rangle_F \qquad \forall C' \in \R^{n \times n}
\end{equation}
for the Frobenius inner product $\langle \cdot, \cdot \rangle_F$. The following theorem shows that $\partial \GG(C)$ is precisely given by the set of optimal plans \citep[Prop.~9.2]{peyre2019computational}, which in general may be non-unique.

\begin{theorem}
    For the mapping $C \mapsto \GG(C)$, we have
    \begin{equation}
        \partial \GG(C) = \left\{ \gamma^* \in \Pi(\mu_n, \nu_n) : \gamma^* \in \argmin_{\Pi(\mu_n, \nu_n)} K_c(\gamma)  \right\}.
    \end{equation}
\end{theorem}

Thus, computing a supergradient of $\partial \GG(C)$ requires no more computation than solving the OT problem itself, as it is simply the value of the optimal plan found when solving the linear program.

The (super-)gradient $\nabla_d \widehat{\jointgain}(d)$ then requires us to use the chain rule to compute $\nabla_d \GG(C(d))$:
\begin{align}
    \nabla_d \widehat{\jointgain}(d) &= \sum_{j, k = 1}^n \gamma^*_{jk} \nabla_d C_{jk}(d) \\
    &= \sum_{j,k=1}^n \gamma^*_{jk} \left( \nabla_d y_j(\eta, \theta, d)^T  \partial_2 c(\theta_j, y_j, \theta_k', y_k') + \nabla_d y_k'(\eta, \theta, d)^T \partial_4 c(\theta_j, y_j, \theta_k', y_k')  \right)
\end{align}
where the second equality follows from directly computing $\nabla_d C_{jk}(d)$. Here, we use $\partial_2 c$ and $\partial_4c$ to denote the partial derivatives of $c$ with respect to its second and fourth arguments, and $\nabla_d y(\eta, \theta, d) \in \R^{d_y \times d_d}$ is the Jacobian of $y$ with respect to $d$. In practice, $\nabla_d C_{jk}(d)$ can be computed via automatic differentiation when we have a differentiable cost function and a differentiable sampler. In particular, as discussed in \cref{subsect:est_opt}, using the noise outsourcing lemma~\citep{kallenberg1997foundations},  for any fixed noise distribution with appropriate reference measure, $q(\eta)$, there exists (subject to weak assumptions) a function $h$ such that $y(\eta, \theta, d) = h(\eta; \theta, d)$ for $\eta \sim q(\eta)$. If we further assume that $d \mapsto h(\eta; \theta, d)$ is differentiable for our chosen $q(\eta)$, we may use automatic differentiation to compute $\nabla_{d}\; y(\eta, \theta, d)$. We refer to \citet[Chapter~9]{peyre2019computational} for a further discussion of the differentiability of optimal transport discrepancies. While in principle $\nabla_d \widehat{\jointgain}(d)$ is merely a supergradient, this is sufficient for the purposes of performing stochastic gradient-ascent based procedures on the \gls{mti}.

\section{Transport Dependencies and Mutual Information}
\label{sect:appendix_theory}

In this section, we give a proof for \cref{thm:eig_bounds_mtd}, as well as an extended discussion of this result. In addition, we show that the total variation distance between $p(\theta, y \mid d)$ and $p(\theta) p(y \mid d)$ may be obtained as a special case of the \gls{mti}, and provide an upper bound analogous to \cref{thm:eig_bounds_mtd}.

\subsection{The Euclidean Case}
While the \gls{mti} can be defined for general Polish spaces, we work under a Euclidean assumption throughout this section to facilitate the analysis, and also because this is a highly practically relevant scenario. We turn our attention towards bounding the transport dependencies by the mutual information. The key assumption we rely on is a strong log-concavity assumption.

\begin{definition}[Strong Log-Concavity]
    Let $p \in C^2(\R^m)$ be a twice continuously differentiable probability density function. We say that $p$ is strongly log-concave if there exists $\lambda > 0$ such that for all $x$ with $p(x) > 0$, we have
    \begin{equation}
        -\nabla_x^2 \log p(x) \succeq \lambda I_n.
    \end{equation}

    The greatest such $\lambda$ is called the \textit{parameter of log-concavity} for $p(x)$.
\end{definition}

A generalized form of Talagrand's inequality yields an upper bound on the \gls{mti} in terms of the mutual information \citep[Section~9.3]{villani2008optimal}, \citep[Theorem~4.1]{blower2003gaussian}.

\begin{theorem}[Talagrand's Inequality] \label{thm:talagrand}
Suppose $\XX = \R^m$ is a Euclidean space and $p(x), q(x)$ are two probability densities over $\XX$. Suppose $p(x)$ is strongly log-concave with parameter $\lambda$ and ${\sf KL}[q(x) || p(x)] < \infty$. For the cost function $c(x, x') = |x - x'|^2$, we have
\begin{equation}
    {\sf OT}_c[p(x), q(x)] \leq 2 \lambda^{-1} {\sf KL}\left[ q(x) \mid\mid p(x) \right].
\end{equation}
\end{theorem}

Using Talagrand's inequality, we may relate our \gls{mti} and other transport discrepancies to the mutual information. We provide a more formal statement of \cref{thm:eig_bounds_mtd} here, as well as a proof.

\begin{theorem}
    Suppose $\Theta = \R^n$ and $\YY = \R^m$. 
    \label{thm:wig_lessthan_eig}
    ~\begin{enumerate}
        \item Assume the prior $p(\theta)$ is strongly log-concave with parameter $\lambda_\theta$. For $c(\theta, \theta') = |\theta - \theta'|^2$, we have
        \begin{equation}
            \lambda_\theta \thetagain(d) \leq 2  \EIG(d).
        \end{equation}
        \item Assume the marginal $p(y \mid d)$ is strongly log-concave with parameter $\lambda_{y \mid d}$. For $c(y, y') = |y - y'|^2$, we have
    \begin{equation}
        \lambda_{y \mid d} \ygain(d) \leq 2 \EIG(d).
    \end{equation}
        \item Let $\eta > 0$ be given. When both the prior and likelihood satisfy these assumptions, under cost $c(\theta, \theta', y, y') = \eta |\theta - \theta'|^2 + |y - y'|^2$, for $\lambda = \max\{\lambda_\theta / \eta, \lambda_{y \mid d} \}$ we have
    \begin{equation}
        \lambda \jointgain(d) \leq 2 \EIG(d).
    \end{equation}
    \end{enumerate}
\end{theorem}
\begin{proof}
    We start with the first claim. If $\EIG(d)$ is infinite, then there is nothing to show. When $\EIG(d)$ is finite and $p(\theta)$ is strongly log-concave, by \cref{thm:talagrand} we have
    \begin{equation}
        {\sf OT}_c[p(\theta \mid y, d), p(\theta)] \leq 2 \lambda_\theta^{-1} {\sf KL}[p(\theta \mid y, d) || p(\theta)].
    \end{equation}
    Integrating both sides of this with respect to $p(y \mid d)$ yields
    \begin{align}
        \thetagain(d) &= \int_\YY {\sf OT}_c[p(\theta \mid y, d), p(\theta)] p(y \mid d) \d y \\
        &\leq 2 \lambda_\theta^{-1} \int_\YY {\sf KL}[p(\theta \mid y, d) || p(\theta)] p(y \mid d) \d y \\
        &= 2 \lambda_\theta^{-1} \EIG(d)
    \end{align}
    as claimed. The proof for the second claim is analogous. 
    
    For the last claim, consider $c_{\Theta, \eta}(\theta, \theta') = \eta |\theta - \theta'|^2$. As Lemma \eqref{lemma:mtd_lessthan} applies in arbitrary separable Hilbert spaces, we obtain
    \begin{align}
        \jointgain(d) &\leq \min\{ \TT_{c_\Theta, \eta}^{(\theta)}(d), \TT_{c_\YY}^{(y)}(d) \} \\
        &= \min\{ \eta \TT_{c_\Theta}^{(\theta)}(d), \TT_{c_\YY}^{(y)}(d) \} \\ 
        &\leq 2 \min\{\eta \lambda_\theta^{-1}, \lambda_{y\mid d}^{-1} \} \EIG(d).
    \end{align}
    where the final inequality follows from the first two claims. Rearranging this inequality yields the result.
\end{proof}

While the bounds in this previous section apply for Euclidean spaces, generalizations of Talagrand's inequality exist for general metric spaces \citep{gozlan2010transport}, although resulting in a more complex relationship between information-theoretic quantities and optimal transport divergences.

\subsection{Total Variation and Hamming Distance}

One particular special case that may be of interest is when the cost function is defined via the Hamming distance. In this case, we obtain the total variation distance between $p(\theta, y \mid d)$ and $p(\theta) p(y \mid d)$ as a special case of our \gls{mti} framework, and moreover, obtain an upper bound in terms of the mutual information. However, we note that as this cost function is not differentiable, the \gls{mti} under this cost function cannot be directly optimized with gradient-based methods.

\begin{theorem}
    Suppose $\Theta \times \YY$ is a metric space. Consider the Hamming distance $c_H(\theta, y, \theta', y') = \mathbbm{1}[ (\theta, y) \neq (\theta', y')]$. Then,
    \begin{equation}
        \TT_{c_H}(d) = |p(\theta, y \mid d) - p(\theta)p(y\mid d)|_{\textrm{TV}} = \sup_{A \in \BB} \left\{ \int_A p(\theta, y\mid d) \d \theta \ d y - \int_A p(\theta) p(y \mid d) \d \theta \d y\right\}
    \end{equation}
    where $\BB$ is the set of all Borel subsets of $\Theta \times \YY$ and $| \cdot |_{\textrm{TV}}$ is the total variation metric. Moreover,
    \begin{equation}
        \TT_{c_H}(d) \leq \sqrt{\frac{1}{2} \EIG(d)}.
    \end{equation}
\end{theorem}
\begin{proof}
    It is well-known that the Hamming cost in the OT problem yields the total variation distance \citep[Lemma~2.20]{massart2007concentration}. By the Csisz\'ar-Kullback-Pinsker inequality \citep{pinsker1964information, csiszar1967information, kullback1967lower, gozlan2010transport}, we immediately obtain the desired upper bound.
\end{proof}

\section{Linear-Gaussian Model}
Here we investigate the behavior of the gain $\jointgain(d)$ in the linear-Gaussian setting under quadratic costs. In particular, we may obtain a closed-form solution for the \gls{mti} in this case, allowing for a direct comparison against the \gls{mi}. See \cref{fig:linear_gaussian} for an illustration.

\begin{figure}[h]
    \centering
    \includegraphics[]{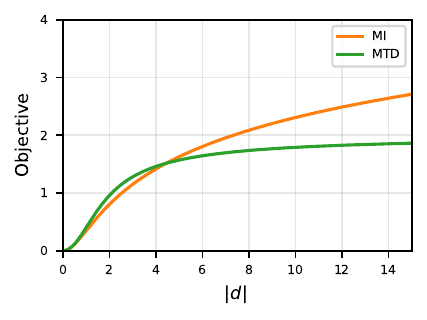}
    \caption{Values of the \gls{mi} and \gls{mti} for the linear-Gaussian model for $\sigma^2_{d,\theta} = 1$.}
    \label{fig:linear_gaussian}
\end{figure}

\begin{theorem}
    Suppose $\theta \in \Theta = \R^n$ has a standard normal prior $p(\theta) = \NN(0, I_n)$, designs are vectors $d \in \DD = \R^n$, and $y \in \YY = \R$ has likelihood $p(y \mid \theta, d) = \NN\left( \langle d, \theta \rangle, \sigma^2_{d,\theta} \right)$. Under the quadratic cost, we have
    \begin{align}
    \jointgain(d) = 2 \bigg(&1 + \sigma^2_{d,\theta} + |d|^2 \label{eqn:mti_linear_gaussian} - \sqrt{1 + (|d|^2 + \sigma^2_{d,\theta})^2 + 2\sqrt{|d|^2 + \sigma^2_{d,\theta}}} \bigg). \nonumber
    \end{align}
    Moreover, $\jointgain(d) \leq 2$. 
    On the other hand, the MI is
    \begin{equation}
        \EIG(d) = \frac{1}{2}\log\left( 1 + |d|^2/\sigma^2_{d,\theta}\right)
    \end{equation}
    which is unbounded as $\sigma^2_{d,\theta}\to 0$.
\end{theorem}
\begin{proof}
    
Define $s = |d|^2 + \sigma^2_{d,\theta}$. Under this setting we may explicitly calculate the joint and product of marginals as
\begin{equation}
    p(\theta, y \mid d) = \NN\left(0, \Sigma_J  \right) \qquad \Sigma_J = \begin{bmatrix} I & d \\ d^T & s \end{bmatrix}
\end{equation}
\begin{equation}
    p(\theta) p(y \mid \xi) = \NN\left(0, \Sigma_P \right)  \qquad \Sigma_P = \begin{bmatrix} I & 0 \\ 0 & s \end{bmatrix}.
\end{equation}

Under quadratic costs, $\jointgain(d)$ is the squared 2-Wasserstein between these distributions, which admits a closed-form via the (squared) Bures-Wasserstein metric
\begin{equation}
    \jointgain(d) = \tr\left(\Sigma_J + \Sigma_P - 2 \sqrt{\Sigma_P^{1/2} \Sigma_J \Sigma_P^{1/2}}\right).
\end{equation}

Observe that $ \tr(\Sigma_J) = \tr(\Sigma_P) = n + s$ and thus we turn our attention to the term $B = \sqrt{\Sigma_P^{1/2} \Sigma_J \Sigma_P^{1/2}}$. Since $\Sigma_P$ is diagonal its square-root is simply
\begin{equation}
    \Sigma_P^{1/2} = \begin{bmatrix} I & 0 \\ 0 & \sqrt{s} \end{bmatrix}
\end{equation}
and an explicit calculation of $B^2$ yields
\begin{equation}
    B^2 = \Sigma_P^{1/2}\Sigma_J \Sigma_P^{1/2} = \begin{bmatrix} I & \sqrt{s}\xi  \\ \sqrt{s}\xi^T & s^2 \end{bmatrix}.
\end{equation}

We require $\tr(B)$. Let $(\lambda_i)$ be the eigenvalues of $B^2$, so that $\text{tr}(B) = \sum_i \sqrt{\lambda_i}$. Using the expression for the determinant of a block matrix, we see that the characteristic polynomial of $B^2$ is 
\begin{equation}
    p(\lambda) = (1-\lambda)^n \left(s^2 - \lambda - \frac{s r^2}{1 - \lambda} \right) = (1 - \lambda)^{n-1} \left( \lambda^2 - (1+s^2)\lambda + s  \right).
\end{equation}
Thus, $B^2$ has eigenvalues $\lambda_1 = \lambda_2 = \dots = \lambda_{n-2} =1$ and eigenvalues $\lambda_{n-1}, \lambda_n$ which are the roots of the quadratic part. In particular, $\lambda_{n-1} + \lambda_n = 1 + s^2$ and $\lambda_{n-1}\lambda_n = s$. This yields
\begin{equation}
    \sqrt{\lambda_{n-1}} + \sqrt{\lambda_n} = \sqrt{1 + s^2 + 2\sqrt{s}}
\end{equation}
\begin{equation}
    \tr(B) = (n-1) + \sqrt{1 + s^2 + 2\sqrt{s}}.
\end{equation}

Putting everything together via the additivity of the trace, we have
\begin{equation}
    \jointgain(d) = 2 \bigg(1 + \sigma^2_{d,\theta} + |d|^2 - \sqrt{1 + (|d|^2 + \sigma^2_{d,\theta})^2 + 2\sqrt{|d|^2 + \sigma^2_{d,\theta}}} \bigg).
\end{equation}
Using a computer algebra system one can verify that $\jointgain(d) \leq 2$ and that $\lim_{|d|\to\infty} \jointgain(d) = 2$.

On the other hand, the mutual information between two Gaussians admits a closed form. This immediately gives
\begin{equation}
    \EIG(d) = \frac{1}{2} \log\left(1 + \frac{|d|^2}{\sigma^2_{d,\theta}}\right).
\end{equation}
which is unbounded as $\sigma^2_{d,\theta} \to 0$.
\end{proof}

\section{Experiment Details}
\label{sect:appendix_experiment_details}

This section provides additional details for our experiments. Unless specified otherwise, our experiments in \cref{sect:experiments} were performed with the following settings. Experiments were performed on an Apple M4 Pro chip with 24 GB of unified memory and a 14-core CPU and primarily implemented in \texttt{pytorch} \citep{paszke2019pytorch}.

Designs are optimized for $250$ gradient steps with a learning rate of $2$e-$2$ using the Adam optimizer \citep{kingma2014adam}. For \gls{mti}, we draw $1\,000$ samples from $p(\theta, y \mid d, h_t) = p(\theta \mid h_t) p(y \mid \theta, d)$ per gradient step, shuffled following \cref{sect:appendix_mtd_details} to yield samples from $p(\theta)p(y \mid d)$. For PCE, we draw $N = 1\,000$ samples $(\theta_n, y_n) \iid p(\theta, y \mid d, h_t)$ to form the approximation \citep{foster2020unified}
\begin{equation}
    \EIG^{(t)}(d) \approx \frac{1}{N} \sum_{n=1}^N \left[ \log p(y_n \mid \theta_n, d) - \log\left(\frac{1}{L+1} \left( p(y_n \mid \theta_n, d) + \sum_{\ell=1}^L p(y_n \mid \theta_{\ell, n}, d)\right)\right)\right].
\end{equation}
where $L=1\,000$ and $\theta_{\ell, n} \iid p(\theta \mid h_t)$. 

\cref{fig:leading} and \cref{fig:2d_lf_weighting} are plotted with additional Gaussian smoothing for visualization purposes.

\subsection{Location Finding}

In the location finding task, our goal is to estimate the spatial location of $K \geq 1$ sources $\theta_k \in \R^{d_{\theta}}$. The number of sources $K$ is assumed to be known, so that $\theta = (\theta_1,  \theta_2, \dots, \theta_K)$. Each source $\theta_k$ emits a signal which decays according to an inverse square law. In each step, a sensor is placed at a location $d \in \R^{d_\theta}$ which records a noisy measurement $y \in \R$ of the total signal intensity at the sensor location \citep{sheng2005maximum}. We note that this task has become a standard benchmark for BED \citep{foster2021deep, ivanova2021implicit, blau2022optimizing, iollo2024bayesian}.

\begin{figure}[t]
    \centering
    \includegraphics[]{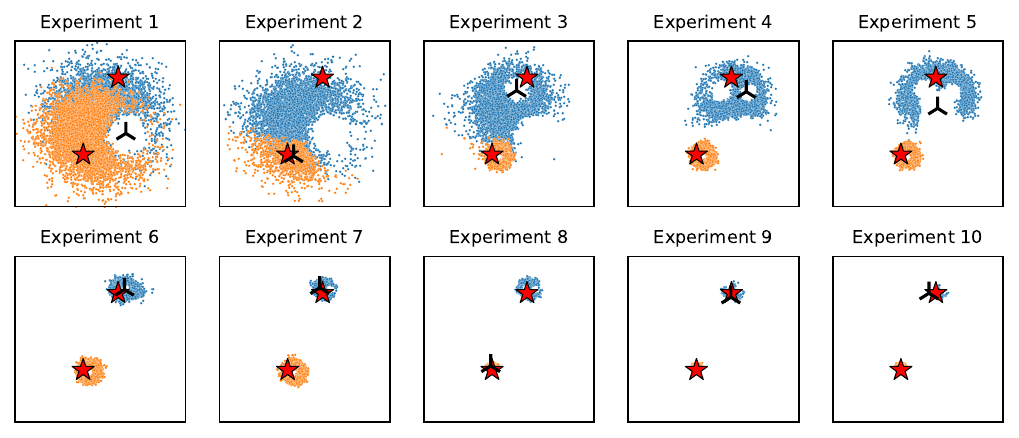}
    \caption{An illustration of the LF problem. Red stars indicate the true, unknown source locations $\theta_1, \theta_2$. At each experimental iteration, a measurement location (black triangle) is selected by maximizing the \gls{mti}. Posterior samples $\theta \sim p(\theta \mid h_t)$ (orange and blue points) depict the updated beliefs about the source positions after each measurement. The \gls{mti} adaptively selects measurement locations that rapidly determine both sources.}
    \label{fig:lf_viz}
\end{figure}

The total (noiseless) intensity at a location $d$ is given by
\begin{equation}
    \mu(\theta, d) = b + \sum_{k=1}^K \frac{\alpha_k}{m + |\theta_k - d|^2}.
\end{equation}

Here, $b, \alpha_k, m \geq 0$ are known constants. The variable $b$ represents a background signal level and $m$ controls the (inverse) maximum signal strength. We assume an independent standard Gaussian prior over each $\theta_k$,
\begin{equation}
    p(\theta_k) = \NN(\theta_k \mid 0, I_{d_\theta}) \qquad k = 1, \dots, K.
\end{equation}
We further assume that we observe the logarithm of the total signal intensity with Gaussian noise, i.e.,
\begin{equation}
    y \mid \theta, d \sim \NN\left(y \mid \log \mu(\theta, d), \sigma^2\right)
\end{equation}
where $\sigma^2 > 0$ is assumed to be known. In all experiments we follow prior work \citep{foster2021deep} and set $b = 0.1, \alpha_k = 1, m  =10^{-4}, \sigma^2=0.25$. Further, we consider $K=2$ sources in $d_\theta = 2$ dimensions, so that $\theta$ is four dimensional.

\paragraph{Inference.} For inference, we use the NUTS MCMC sampler implemented in \texttt{pymc3} \citep{salvatier2016probabilistic} with $1$e$4$ warm-up steps and four independent chains to draw a total of $1$e$5$ posterior samples at each experiment iteration. In LF, posteriors are complex and multi-modal, necessitating accurate (but expensive) inference. See \cref{fig:lf_viz} for an illustration.

\paragraph{RMSE.} During posterior sampling, the model exhibits non-identifiability with respect to the ordering of the two sources, $\theta_1$ and $\theta_2$. As a result, the correspondence between estimated and true sources may be swapped across samples. To address this when computing the RMSE, we evaluate both possible orderings of each posterior sample and use the ordering that yields the lower RMSE.

\paragraph{Weighting Function.} In \cref{sect:experiments_weighting} we modify the cost function in \gls{mti} using a weighting function. In particular, this takes the form
\begin{equation}
    w(\theta) = b + g(\theta_1 - \mu) + g(\theta_2 - \mu)
\end{equation}
where $b=1$ is a bias and $\mu = (1.5, -1.5)$ controls the location of the bump, and
\begin{equation}
    g(x) = k\left(1 - \textrm{sigmoid}\left(s \left(\frac{|x|^2 - \alpha^2}{\beta^2 - \alpha^2} \right) \right) \right)
\end{equation}
is a bump-like function. Here, $\beta=1$ approximately controls the radius of its support, $\alpha=0.5$ controls an inner radius where $g$ is approximately maximized, $s=0.3$ controls the slope of the bump, and $k=1$e$4$ controls its amplitude. See \cref{fig:bump_fxn} for a visualization. The weighting $w(\theta)$ is selected to depend only on $\theta$ which is drawn from the joint.

\begin{figure}[h]
    \centering
    \includegraphics[]{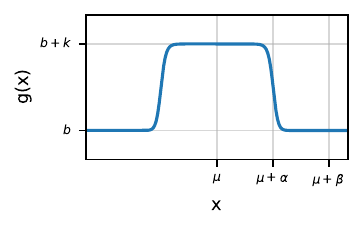}
    \caption{A visualization of the bump function $g(x)$.}
    \label{fig:bump_fxn}
\end{figure}

\subsection{Constant Elasticity of Substitution (CES)}

The Constant Elasticity of Substitution (CES) model, arising from behavioral economics, asks a participant to compare two baskets $d_1, d_2 \in [0, 100]^3$ consisting of various amounts of three different goods. Given two baskets, the participant provides a scalar response $y \in [0, 1]$ indicating subjective preference between the baskets. This model has previously served as a benchmark for several recent BED works \citep{foster2019variational, foster2020unified, blau2022optimizing, iollo2024pasoa}. 

\looseness=-1
The experimental goal is to choose a design $d = (d_1, d_2) \in [0, 100]^6$ consisting of two baskets in order to infer the participant's latent utility function. This utility function is assumed to be parametrized by $\theta = (\rho, \alpha, u)$ where $\rho \in [0, 1], \alpha \in \Delta_3, u \in \R_{\geq 0}$ and $\Delta_3$ is the $3$-simplex. Thus, $\theta \in \R^5$ is a five-dimensional unknown parameter of interest.

Following previous work, we assume the following priors:
\begin{equation}
    \rho \sim \beta(1, 1)
\end{equation}
\begin{equation}
    \alpha \sim \textrm{Dir}(1, 1, 1)
\end{equation}
\begin{equation}
    \log u \sim \NN(1, 3).
\end{equation}

The likelihood for the participant's response is modeled as
\begin{equation}
    U(d) = \left(\sum_{i=1}^3 d_i^\rho \alpha_i \right)^{1/\rho}
\end{equation}
\begin{equation}
    \mu = (U(d_1) - U(d_2)) u
\end{equation}
\begin{equation}
    \sigma = (1 + |d_1 - d_2|)\tau u
\end{equation}
\begin{equation}
    \eta \sim \NN(\mu, \sigma^2)
\end{equation}
\begin{equation}
    y = \begin{cases}
        \sm(\eta) & \epsilon < \sm(\eta) < 1 - \epsilon \\
        \epsilon  & \sm(\eta) \leq \epsilon \\
        1 - \epsilon & \sm(\eta) \geq 1 - \epsilon
        \end{cases}
\end{equation}

where $\epsilon = 2^{-22}, \tau = 5\mathrm{e}{-3}$ are fixed constants and 
\begin{equation}
    \textrm{sigmoid}(x) = \frac{1}{1 + e^{-x}}
\end{equation}
is the usual sigmoid function. Thus, the participant's response depends on the difference in utilities $U(d_1) - U(d_2)$ between the two baskets.

\paragraph{Log-Likelihood.} The CES model can present numerical challenges for methods that require evaluating the likelihood of observations (e.g., PCE). We follow the recommendations in \citep{foster2020unified, iollo2024pasoa} to evaluate this quantity. In particular, since $y$ is censored, we have that
\begin{equation}
    p(y \mid \theta, d) = p_0 \delta[y = \epsilon] + p_1 \delta[y = 1 - \epsilon] + (1 - p_0 - p_1) q(y \mid \theta, d)
\end{equation}
is a mixture of Dirac deltas at the boundaries and a density on the interior, where the density
\begin{equation}
    q(y \mid \theta, d) = \frac{1}{\sigma y(1-y) \sqrt{2 \pi}} \exp\left( -\frac{(\textrm{logit}(y) - \mu)^2}{2 \sigma^2}\right)
\end{equation}
represents a logit-normal distribution away from the boundary with $\textrm{logit}(y) = \sm^{-1}(y) = \log(y / (1 - y))$. The quantities $p_0, p_1$ are defined via
\begin{equation}
    p_0 = \Phi\left( \frac{\textrm{logit}(y) - \mu}{\sigma} \right)
\end{equation}
\begin{equation}
    p_1 = 1 - \Phi \left( \frac{\textrm{logit}(1 - \epsilon) - \mu}{\sigma} \right)
\end{equation}
where $\Phi$ is the standard normal CDF. When $p_0, p_1 \ll 1$, computing their logarithms becomes challenging, in which case we approximate $\Phi$ by a first-order Taylor expansion, i.e.,
\begin{equation}
    \Phi(x) \approx \frac{1}{-x \sqrt{2\pi}} \exp(-x^2 /2) \qquad x \ll -1 
\end{equation}
\begin{equation}
    1 - \Phi(x) \approx \frac{1}{x \sqrt{2\pi}} \exp(-x^2 /2) \qquad x \gg 1.
\end{equation}

We perform additional clipping as necessary before taking logarithms in our implementation to further improve stability.

For inference, we perform importance resampling \citep{doucet2001introduction} where $1$e$7$ proposal samples are drawn from the prior, weighted according to the likelihood, and $1$e$5$ proposal samples are re-drawn with replacement with probability proportional to their weight.

\section{Additional Experiments}

\begin{figure*}[ht]
    \centering
    \includegraphics[]{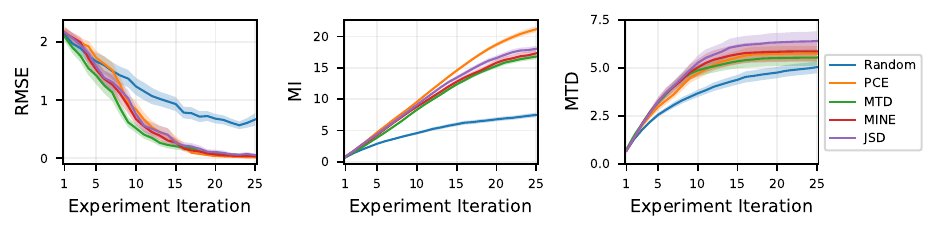} %
    \caption{Quantitative results for the location finding problem, averaged across 25 seeds ($\pm$ one standard error). \gls{mti} achieves the lowest RMSE across most of the rollout, even though its corresponding \gls{mi} values are slightly lower. All methods obtain comparable \gls{mti} scores within one standard error, reflecting the high variance of this quantity.}
    \label{fig:2d_lf_highsamp}
\end{figure*}

\subsection{Location Finding}
One appeal of the \gls{mti} is that it is naturally an implicit method, as it only relies on our ability to draw samples. To highlight this, we compare against two methods for \gls{mi}-based experimental design in implicit settings: MINEBED \citep{kleinegesse2020bayesian} and JSD \citep{kleinegesse2021gradient}.

MINEBED is based on the NWJ \citep{nguyen2010estimating} lower bound for the mutual information:
\begin{equation}
    \EIG(d) \geq  \E_{p(\theta, y \mid d)}\left[T(\theta, y) \right] - e^{-1} \E_{p(\theta)p(y\mid d)}\left[\exp( T(\theta, y) )\right]
\end{equation}
where $T: \Theta \times \YY \to \R$ is an arbitrary measurable function. This lower bound is, in fact, tight. In practice, we take $T(\theta, y) = T_\psi(\theta, y)$ to be a neural network parametrized by $\psi$, yielding a strict lower bound to $\II(d)$ when the considered class of neural networks does not contain the true optimum. In the continuous setting, the network parameters $\psi$ and design $d$ may be optimized simultaneously via gradient-based procedures. 

On the other hand, JSD is based on a lower bound of the Jensen-Shannon divergence \citep{hjelm2018learning}
\begin{align}
    \JJ(d) &= \frac{1}{2} {\sf KL}\left[p(\theta, y \mid d) || m(\theta, y \mid d) \right] + \frac{1}{2} {\sf KL}\left[ p(\theta)p(y \mid d) || p(\theta) p(y \mid d) \right] \\
    &= \log 2 - \frac{1}{2}\E_{p(\theta, y \mid d)}\left[ \log\left(1 + \frac{p(\theta)p(y \mid d)}{p(\theta, y \mid d)} \right)\right] - \frac{1}{2}\E_{p(\theta)p(y \mid d)}\left[ \log\left(1 + \frac{p(\theta, y \mid d)}{p(\theta)p(y \mid d)} \right) \right] \\
    &\geq \log2 + \frac{1}{2}\left( \E_{p(\theta, y \mid d)}\left[-\log(1 + \exp(-T(\theta, y)) \right] - \E_{p(\theta)p(y\mid d)}\left[ \log(1 + \exp T(\theta,y)) \right]\right)
\end{align}

where $m(\theta, y \mid d) = \frac{1}{2}\left(p(\theta, y \mid d) + p(\theta)p(y \mid d) \right)$ is a mixture distribution. Again we take $T(\theta, y) = T_\psi(\theta, y)$ to be a neural network, yielding a potentially strict lower bound. \citet{kleinegesse2021gradient} motivate the JSD as an objective for OED by noting that it behaves similarly to the mutual information while potentially offering more stable training, as there is no exponential of the network unlike in MINEBED.

For both MINEBED and JSD, we parametrize $T_\psi$ by an MLP with two hidden layers of size 200, trained at a learning rate of $3$e$-3$ using Adam \citep{kingma2014adam}. Designs are simultaneously optimized via Adam, but at a higher learning rate of $2$e$-2$.

\paragraph{Results.} In \cref{fig:2d_lf_highsamp}, we report the RMSE, \gls{mi}, and \gls{mti} values for our method, PCE, MINE, and JSD, with all designs optimized using five random restarts. The \gls{mti}-based approach consistently achieves the lowest RMSE across most of the experimental rollout, even though it attains slightly lower \gls{mi} values, which is expected, since it does not explicitly optimize this objective. All methods obtain comparable \gls{mti} scores; this is largely attributable to the high variance of the \gls{mti} estimate and to the nature of the location-finding task, where, once the posterior becomes highly concentrated, the set of effective designs narrows substantially, so that any design within this region yields similarly strong \gls{mti} values.

\subsection{CES}

\begin{figure*}[!ht]
    \centering
    \includegraphics[]{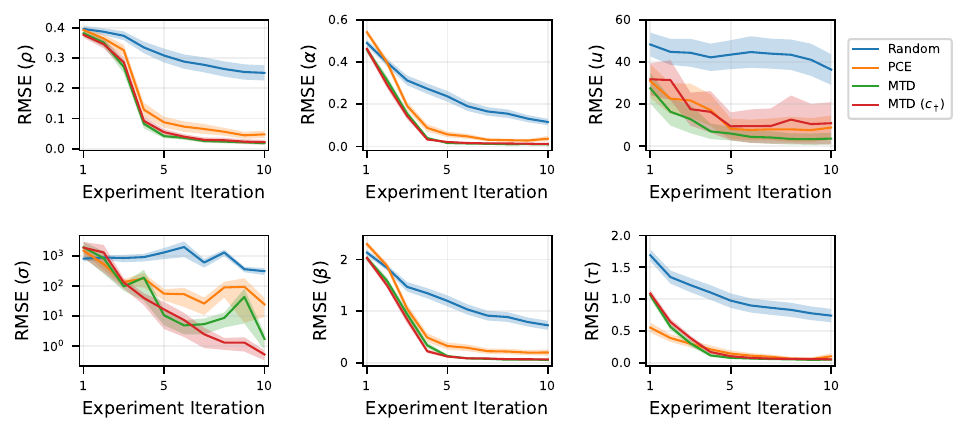}
    \caption{Quantitative results for the CES problem, averaged over 50 seeds (± one standard error). On the original scale (top row), \gls{mti} achieves the lowest RMSE, while the transformed variant $\TT_{c_\dagger}$ performs slightly worse on $u$. On the transformed scale (bottom row), $\TT_{c_\dagger}$ attains the lowest RMSEs, reflecting its alignment with the evaluation variables.}
    \label{fig:ces_convergence}
\end{figure*}

\begin{figure*}[!ht]
    \centering
    \includegraphics[]{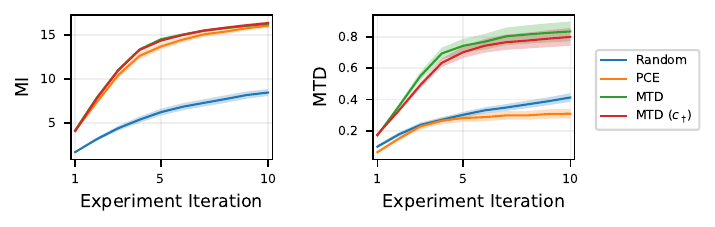}
    \caption{\gls{mi} and \gls{mti} values for the CES problem, averaged over 50 seeds ($\pm$ one standard error). \gls{mti} achieves similar \gls{mi} values to PCE, while PCE attains worse \gls{mti} values than random designs.}
    \label{fig:ces_gain}
\end{figure*}

To supplement the results in \cref{sect:experiments} for the CES problem, we provide additional figures visualizing our results.

\cref{fig:ces_convergence} plots the RMSE values for the CES problem across the experimental iterations. When evaluating performance on the original parameter scale (top row of \cref{tab:ces_50seeds}), \gls{mti} achieves the lowest RMSE, while the transformed variant $\TT_{c_\dagger}$ performs slightly worse on $u$. This is expected, as the transformed cost emphasizes errors in a different coordinate system. Conversely, when performance is assessed in the transformed space (bottom row), $\TT_{c_\dagger}$ attains the lowest RMSEs, illustrating that defining the cost in terms of the relevant variables can effectively guide design selection toward the aspects of the parameters that are most important for downstream evaluation. These results highlight the flexibility of \gls{mti}: by choosing an appropriate cost, either directly or via transformations, experimenters can align the design criterion with the metric that truly matters for their task.

\cref{fig:ces_gain} shows the \gls{mi} and \gls{mti} values for CES. We observe that designs optimized using the \gls{mti} achieve comparable \gls{mi} values to PCE. On the other hand, the designs produced by PCE yield low \gls{mti} values, performing worse than random. Notably, this serves as a verification of our \cref{thm:eig_bounds_mtd}, which can be interpreted as showing that designs which have high \gls{mti} tend to have large \gls{mi} as well, but not the converse.